\documentclass{article}

     \usepackage[nonatbib,preprint]{neurips_2022}

\usepackage[utf8]{inputenc} 
\usepackage[T1]{fontenc}    
\usepackage{hyperref}       
\usepackage{url}            
\usepackage{booktabs}       
\usepackage{amsfonts}       
\usepackage{nicefrac}       
\usepackage{microtype}      
\usepackage[dvipsnames]{xcolor}

\title{Star algorithm for NN ensembling}

\author{
  Sergey ~Zinchenko\\
  Novosibirsk State University\\
  Russian Federation\\
  \texttt{zinch.s.e@gmail.com} \\
   \And
  Dmitry ~Lishudi \\
  HSE University,\\
  Russian Federation\\
  \texttt{dlishudi@hse.ru} \\
}

\usepackage{amsmath, amsfonts, amssymb, amsthm, mathtools}
\usepackage{icomma}
\usepackage{etoolbox}

\usepackage[dvipsnames]{xcolor} 
\usepackage{hyperref}
\newcommand\fnurl[2]{%
  \href{#2}{#1}\footnote{\url{#2}}}%

\usepackage[
backend=biber,
style=alphabetic,
sorting=ynt
]{biblatex}
\newtheorem{theorem}{Theorem}[section]
\newtheorem{definition}[theorem]{Definition}

\newtheorem{lemma}[theorem]{Lemma}
\newtheorem{corollary}[theorem]{Corollary}
\DeclareMathOperator*{\Prob}{\mathbb{P}}
\DeclareMathOperator*{\E}{\mathbb{E}}
\DeclareMathOperator*{\HE}{\widehat{\mathbb{E}}}

\DeclareMathOperator*{\argmin}{arg\,min}
\newcommand{\wave}{\widetilde}

\usepackage{wrapfig}
\usepackage{blindtext}
\usepackage{multicol}
\usepackage{svg}
\usepackage{graphicx}
\usepackage{tikz}
\usetikzlibrary{decorations.pathreplacing}
\usetikzlibrary{fadings}
\usetikzlibrary{positioning,shadows,arrows,trees,shapes,fit}
\usetikzlibrary{calc,intersections}
\usepackage{lipsum}
\usepackage{algorithm2e}
\usepackage{caption}

\usepackage{array}
\newcolumntype{P}[1]{>{\centering\arraybackslash}p{#1}}

\usepackage[dvipsnames]{xcolor}

\begin{document}

\maketitle

\begin{abstract}
    Neural network ensembling is a common and robust way to increase model efficiency. In this paper, we propose a new neural network ensemble algorithm based on Audibert's empirical star algorithm. We provide optimal theoretical minimax bound on the excess squared risk. Additionally, we empirically study this algorithm on regression and classification tasks and compare it to most popular ensembling methods.
\end{abstract}

\section{Introduction}

Deep learning has been successfully applied to many types of problems and has reached the state-of-the-art performance. Deep learning models have shown good results in regression analysis and time series forecasting \cite{qiu2014ensemble}, computer vision \cite{he2016deep}, as well as in natural language processing \cite{otter2020survey} and other areas. In many complex problems, such as the Imagenet competition \cite{deng2009imagenet}, the best results are achieved by ensembles of neural networks, that is, it is often useful to combine the predictions of multiple neural networks to create a new one. The easiest way to ensemble multiple neural networks is to average their predictions \cite{drucker1994boosting}. As shown in work \cite{kawaguchi2016deep}, the number of local minima grows exponentially with the number of parameters. And since modern neural network training methods are based on stochastic optimization, two identical architectures optimized with different initializations will probably converge to different solutions. Such a technique for obtaining neural networks with subsequent construction of an ensemble by majority voting or averaging is used, for example, in article \cite{caruana2004ensemble}.

In addition to the fact that the class of deep neural networks has a huge number of local minima, it is also non-convex. It was shown in work \cite{lecue2009aggregation} that for the procedure of minimizing the empirical risk in a non-convex class of functions, the order of convergence is not optimal. In fact, most modern neural network training methods do just that: they minimize the mean value of some error function on the training set. J.-Y. Audibert proposed the star procedure method, which has optimal rate of convergence of excess squared risk \cite{a07-mixture}. Motivated by this observation and the huge success of ensembles of neural networks, we propose a modification of the star procedure that will combine the advantages of both methods. In short, the procedure we propose can be described as follows: we run $d$ independent learning processes of neural networks, obtaining empirical risk minimizers $\widehat{g}_1,\, \dots,\,\widehat{g}_d$, freeze their weights, then we initialize a new model and connect all $d + 1$ models with a layer of convex coefficients, after that we start the process of optimizing all non-frozen parameters. This whole procedure can be viewed as a search for an empirical minimizer in all possible $d$-dimensional simplices spanned by $d$-minimizers and a class of neural networks.
As is known, the minimization of the empirical risk with respect to the convex hull is not optimal in the same way as with respect to the original class of functions. Our method, however, minimizes over some set intermediate between the original class of functions and its convex hull, allowing us to combine the advantages of model ensembling and the star procedure.

 One can look at this procedure as a new way to train one large neural network with a block architecture, as well as a new way of aggregating models. In this work, we carry out a theoretical analysis of the behavior of the proposed algorithm for solving the regression problem with a class of sparse neural networks, and also check the operation of the algorithm in numerical experiments on classification and regression problems.

In addition to this, we take into account that it is impossible to achieve a global minimum in the class of neural networks, and we consider the situation of imprecise minimization.

The \textit{main results} of our work can be formulated as follows:

\begin{itemize}
    \item A multidimensional modification of the star procedure is proposed.
    \item We prove that the resulting estimate satisfies the exact oracle inequality. It follows from this estimate that the order of convergence of the algorithm (in terms of sample size $n$) for a fixed neural network architecture is optimal. Our results improve over the imprecise oracle inequality in \cite{schmidt2020nonparametric}.
    \item We give an upper bound on the generalization error for the case of approximate empirical risk minimizers, which implies the stability of our algorithm against minimization errors.
    \item Based on our algorithm, we propose a new method for training block architecture neural networks, which is quite universal in terms of procedures. We also propose a new way to solve the aggregation problem.
    \item We illustrate the efficiency of our approach with numerical experiments on real-world datasets.
\end{itemize}

    The rest of this paper is organized as follows.
    
    In Section \ref{related work}, we make an overview of neural network ensembling methods and briefly discuss the advantages of the star algorithm.
    In Section \ref{main_part}, following the Schmidt-Hieber notation \cite{schmidt2020nonparametric}, we define a class of sparse fully connected neural networks and formulate a number of statements from which it follows that the algorithm we proposed has a fast rate of convergence.
    All proofs are attached in additional materials.
    In Section \ref{experiments}, we discuss the implementation of our algorithm, point out a number of possible problems, and suggest several modifications to fix them. It also describes the conditions for conducting numerical experiments and presents some of their results. 
    At the end, we offer two possible views on our procedure: a new way to train block neural networks and a fairly flexible model aggregation procedure.

{\section{Related work}
\label{related work}}
    \subsection{Ensemble strategies}
    The main idea of the ensemble is to train several predictors and build a good metamodel on them.
    There are many techniques for its construction. We present some of them. A more detailed review can be found in the work \cite{ensemble_review}.
    
    \textit{Bagging.}
    
    The first of two stages is the generation of several samples with the same distribution as the training one. The next stage is training multiple models and aggregate their predictions.
    There are cases when the predictions of the constructed models are transferred to another model as new features \cite{kim2002support}. But still, most often, aggregation is performed either by majority voting or by averaging. For example, Kyoungnam Ha et al. showed  that the ensemble method for perceptrons based on bagging works better than the individual neural network \cite{ha2005response}.\\
    
    \textit{Boosting}
    
    Another approach to construct ensembles is boosting. The idea is to build one strong model from several weak models by stepwise additive modeling. It was first applied to random trees to construct a so-called random forest. But it has also been applied to deep learning models as well. For example, in the task of recognizing facial expressions \cite{liu2014facial}, or to improve the predictions of convolutional neural networks \cite{moghimi2016boosted}.
    \clearpage
    
    \textit{Snapshots}
    
    The main problem in aggregation of deep learning models is the cost of training.
    Training even one modern model requires a lot of resources, and the ensemble needs a lot. A snapshot technique \cite{snapshots} and their modification \cite{FGE} have been created to combat this problem. In short, during the learning process, the step length in the gradient descent is cyclically changed. This allows a learner to get into various local minima (parameters of which are stored for subsequent aggregation) and, as a consequence, to build an ensemble using a computational budget comparable with the cost of training one model.
    
    Later, Zhang et al. tried combining this idea together with boosting in \cite{zhang2020snapshot}. The novelty of this work is that the difference between the models is achieved by re-weighting the objects from the training sample, the step length parameter is changed adaptively, and for aggregation, the best model from the entire training cycle is taken, and not the last one.
    
    \textit{Implicit ensembles}
    
    In this approach a single model is trained in such a way as to behave like an ensemble. But it requires a much smaller computational budget for training. This is achieved due to the fact that in implicit ensembles the parameters of the models are shared, and their averaging is returned as predictions. For example, an implicit ensemble is the \textit{DropOut} \cite{dropout} method or the \textit{DropConnect} \cite{wan2013regularization} method. During training, each neuron or connection in the neural network has a chance to collapse, and after training, a neural network is returned, the elements of which are weighted by the probabilities of the presence of each element.
    A similar idea is implemented in the \cite{huang2016deep} \textit{Stochastic depth} method for \cite{he2016deep} residual neural networks. There, the residual blocks are randomly discarded during training, and the transformation goes only through a skip connection.

    \subsection{Star algorithm}
    
    Unlike the ensemble problem, the aggregation problem focuses on building a good predictor in a situation where there are already several ready-made models. The reader is referred to \cite{nemirovski2000topics, tsybakov03} for different types of aggregation. It is important to mention that, in contrast to the two-stage star procedure of Audibert \cite{a07-mixture}, the usual empirical risk minimization procedure among the class of known predictors (or their convex hull) does not necessarily lead to the optimal rate of convergence \cite{lecue2009aggregation}. This result was further developed in \cite{liang2015learning}, where the authors extend the theoretical analysis of the star algorithm to the case of infinite classes using the offset Rademacher’s complexity technique.
    It was also shown in the \cite{vijaykumar2021localization} that these results can be generalized to other loss functions. In
    particular, this means that the star procedure can be applied to more than just regression problems.


\section{Theoretical results}
\label{main_part}
    In this chapter, we will formulate our theoretical results for solving the regression problem with a class of sparse fully connected neural networks using our proposed $Star_d$ procedure.

    We have a $S_n = (\textbf{X}_i,Y_i)^{n}_{i=1}$ sample of i.i.d. input-output pairs $(\textbf{X}_i,Y_i) \in \mathcal{X}\times{\mathcal{Y}}$ distributed according to some unknown distribution $\mathcal{P}$. We also chose a certain family of solutions $\mathcal{F}$.
    Our goal is to build a new predictor $\widehat{f}$ minimizing the excess risk
    $$ 
    \mathcal{E}(\widehat{g}) := \E(\widehat{g} - Y)^2 - 
    \inf_{f \in \mathcal{F}}\E(f - Y)^2.
    $$
    Let $\HE$ denote the empirical expectation operator
    $$ \HE(f) := \frac{1}{n} \sum_{i=1}^n f(\textbf{X}_i)$$
    and call
    $\widehat{g} \in \mathcal{F}$ a \textit{$\Delta$-empirical risk minimizer in $\mathcal{F}$} if the following inequality holds
        $$
          \HE(\widehat{g} - Y)^2 \le \min_{f \in \mathcal{F}} \widehat{\E} (f - Y)^2 + \Delta.
        $$
    We suggest the next two step procedure. In the first, calculate $\big\{\widehat{g}_i\big\}_{i=1}^d$ -- different $\Delta_1$-empirical risk minimizers in $\mathcal{F}$. And then look for a $\Delta_2$-empirical risk minizer in the next set:
    \begin{equation}
                Star_d\big(\mathcal{F}, \widehat{g}_1, \ldots \widehat{g}_d\big) := \bigg\{\sum_{i=1}^d \lambda_i \widehat{g}_i + \underbrace{\big(1 - \sum_{i = 1}^d \lambda_i\big)}_{\lambda} f \;\bigg| \; \lambda_i, \lambda \in [0, 1]; \; f \in \mathcal{F}\bigg\}.
    \end{equation}
    We will call the found function $\widehat{f} = \widehat{f}\left(\mathcal{F},\,d,\,\Delta_1,\Delta_2\right)$ as \textit{$Star_d$ estimator}. The main result of our work is the proof that the proposed estimator has an optimal excess risk convergence rate in the case when $\mathcal{F}$ is a class of fully connected neural networks. It is defined by the choice of the activation function  $\sigma: \mathbb{R} \rightarrow \mathbb{R}$ and the network architecture. We study neural network with activation function ReLu:
    $$\sigma(x) := max(x,0).$$
    For $\textbf{v}=(v_1, \dots,v_r) \in \mathbb{R}^r$ define shifted activation function $\sigma_{\textbf{v}}: \mathbb{R}^r \rightarrow \mathbb{R}^r$:
    
    $$ \sigma_{\textbf{v}}(\textbf{x}) := \big(\sigma(x_i - v_i) \big)_{i=1}^r .$$
    
    The network architecture $(L,\textbf{p})$ consists of a positive integer $L$ called the \textit{number of hidden layers or depth} and a \textit{width vector} $\textbf{p} = (p_0,\dots,p_{L+1}) \in \mathbb{N}^{L+2}$. A neural network with network architecture $(L, \textbf{p})$ is then any function of the form
    \begin{equation}
    \label{def_fcn}
        f(\textbf{x}) = W_L \sigma_{\textbf{v}_L} W_{L-1} \sigma_{\textbf{v}_{L-1}} \dots W_1 \sigma_{\textbf{v}_1} W_0 \textbf{x},
    \end{equation}
    where $W_j$ is a $p_{j+1}\times p_j$ matrix and $\textbf{v}_i \in \mathbb{R}^{p_i}$ is a shift vector.
    
    We will focus on the case when the model parameters satisfy some constraint. Denote $\|W_j\|_{\infty}$ the maximum-entry norm of $W_j$ , $\|W_j\|_0$  the number of non-zero/active entries of $W_j$ then
    the space of network functions with given network architecture and network parameters bounded by one is
    $$
    \mathcal{F}(L, \textbf{p}) := \bigg\{ f \text{of the form \eqref{def_fcn}}: \max_{j=0,\dots,L} \|W_j\|_\infty \vee \|\textbf{v}_j\|_\infty \leq 1 \bigg\}
    $$  
    and the s-sparse networks are given by 
    $$
    \mathcal{F}(L, \textbf{p}, s) := \bigg\{ f \in \mathcal{F}(L,  \textbf{p}): \sum_{j=0}^L \|W_j\|_0 + \|\textbf{v}_j\|_0 \leq s \bigg\}.
    $$
    Let's denote by $\mathcal{N}_\infty(\mathcal{F},\,\varepsilon)$, $\mathcal{N}_2(\mathcal{F},\,\varepsilon)$ the size of the $\varepsilon$-net of $\mathcal{F}$ in the metric space $L_\infty$ and $L_2$, respectively.
    Then from Lemma 5 in \cite{schmidt2020nonparametric} we have
    \begin{equation}
    \label{shmidt_norm}
    \forall f \in \mathcal{F}(L, \,\textbf{p}, s): \, \|f\|_\infty \leq V(L+1)
    \end{equation}
    and 
    \begin{equation}
    \label{shmidt_cover}
    \log \mathcal{N}_2(\mathcal{F}(L,\,\textbf{p}, s),\delta)\leq \log \mathcal{N}_\infty(\mathcal{F}(L,\,\textbf{p}, s),\delta) \leq (s+1) \log (2\delta^{-1}(L+1)V^2),
    \end{equation}
    where
    \begin{equation}
    \label{V}
    V:= \prod_{l=0}^{L+1}(p_l+1). 
    \end{equation}
    
    Also let define the  risk-minimizer in $\mathcal{F}$ and some sets:
    \begin{align}
    \label{set_Hull}
        &Hull_d\big(\mathcal{F}\big) :=  \bigg\{\sum_{i=1}^d \lambda_i (f_i - f) \;\bigg| \; \lambda_i \in [0, 1];\; \sum_{i = 1}^d \lambda_i \le 1;\; f, f_1 \ldots f_d \in \mathcal{F}\bigg\},\\
    \label{set_H}
        &f^* := \argmin_{f \in \mathcal{F}} \E (f(\textbf{X}) - Y)^2, \,\,\,\,\, \mathcal{H} := \mathcal{F} - f^* + Hull_d(\mathcal{F}).
    \end{align}
    Notice, that $Star_d$ estimator $\widehat{f}$ lies in $\mathcal{H}+f^*$. With the introduced notation, one of our main results is stated as follows.

        \begin{theorem}
        Let $\widehat{f}$ be a $Star_d$ estimator and $\mathcal{H}$ be the set defined in \ref{set_H} for $\mathcal{F} = \mathcal{F}(L, \textbf{p}, s)$. The following expectation bound on excess loss holds:
        $$
            \E \mathcal{E}(\hat{f}) \le 2(F'+V(L+1)) \cdot \left[ \frac{K}{n} + M \cdot \frac{c_{\ref{bound_H}}  d \, s \log \big(V L \,n \,d\big)}{n}\right] + 4[\Delta_1 + \Delta_2],
        $$
        where 
        $$ C = \min\left\{\frac{c_{\ref{geom_ineq}}}{4F'}, \frac{c_{\ref{geom_ineq}}}{4V(L+1)(2+c_{\ref{geom_ineq}})}\right\}, \,\,\,
        F' = \sup_\mathcal{F} |Y - f|_\infty,
        $$
        
        \begin{equation}
            \label{eq:km}
            K:=2\left(\sqrt{\sum_{i=1}^n {\xi^2/n}}+C\right),
            \quad
            M:= \sup_{h \in \mathcal{H} \setminus \{0\}}4\frac{\sum_{i=1}^{n}h(\textbf{X}_i)^2\xi_i^2}{C \sum_{i=1}^{n}h(\textbf{X}_i)^2}.
        \end{equation}

    \end{theorem}

    Thus, the expectation of excess risk for $Star_d$ estimator is limited for a fixed architecture neural network to a value of the order 
    $$
    \mathcal{O}\left(\frac{s \, d \, \log (V\,L\,n)}{n}+\Delta_1+\Delta_2\right).
    $$
    
    In order to formulate an upper bound for the excess risk, performed with a high probability, we need to impose some constraints on the class of functions.
    \begin{definition}[Lower Isometry Bound]
    \label{isom_bound}
        Class $\mathcal{F}$ satisfies the lower isometry bound with some parameters $0 < \eta < 1$ and $0 < \delta < 1$ if
        $$
            \mathbb{P}\left(\inf_{f \in \mathcal{F} \setminus \{0\}}\frac{1}{n} \sum_{i=1}^n \frac{f^2(\textbf{X}_i)}{\E f^2} \ge 1 - \eta\right) \ge 1 - \delta
        $$
        for all $n \geq n_0(\mathcal{F},\delta, \eta)$, where $n_0(\mathcal{F},\delta, \eta)$ depends on the complexity of the class.
    \end{definition}
    
Now we are ready to formulate the main theoretical result of the paper.
\begin{theorem}
        Let $\widehat{f}$ be a $Star_d$ estimator and let $\mathcal{H}$ be the set defined in \ref{set_H} for $\mathcal{F} = \mathcal{F}(L,\,\textbf{p},\,s)$.
        Assume for $\mathcal{H}$ the lower isometry bound in Definition \ref{isom_bound} holds with $\eta_{lib} = c_{\ref{geom_ineq}}/4$ and some $\delta_{lib} < 1$. Let $\xi_i = Y_i - f^*({\textbf{X}_i})$.
        Define
        $$
            A := \sup_{h \in \mathcal{H}} \frac{\E h^4}{(\E h^2)^2} \,\text{ and } B := \sup_{\textbf{X}, Y}\E\xi^4.
        $$
        Then there exist 3 absolute constants $c_{\ref{main_th}}' , \tilde{c_{\ref{main_th}}}, c_{\ref{main_th}} > 0$ (which only depend on $c_{\ref{geom_ineq}}$), such that 
        $$
            \Prob\left(\mathcal{E}(\widehat{f}) > 4(D+\Delta_1+\Delta_2)\right)\leq 4(\delta_{lib}+\delta)
        $$
        as long as $n > \frac{16(1-c_{\ref{main_th}}')^2 A}{c_{\ref{main_th}}'^2} \lor n_0(\mathcal{H}, \delta_{lib}, c_{\ref{main_th}}/4)$, 
        where 
        $$
            K:=\left(\sqrt{\sum_{i=1}^n {\xi^2/n}}+2\tilde{c_{\ref{main_th}}}\right), \,\,\,
            M:= \sup_{h \in \mathcal{H} \setminus \{0\}}\frac{\sum_{i=1}^{n}h({\textbf{X}_i})^2\xi_i^2}{\tilde{c_{\ref{main_th}}} \sum_{i=1}^{n}h({\textbf{X}_i})^2},
        $$
        $$
            D:= \max\left(\frac{K}{n} + M \cdot \frac{c_{\ref{bound_H}}  d \, s \log \big(V L \,n \,d\big) + \log \frac{1}{\delta}}{n}, \frac{32\sqrt{AB}}{c_{\ref{main_th}}'} \frac{1}{n}\right) 
        $$
        and $c_{\ref{bound_H}}$ is an independent constant.
\end{theorem}

    That is, with some assumptions on the class of neural networks $\mathcal{F}$, we again obtained the order $\mathcal{O}\left(\frac{\log n \mathcal{N}}{n}\right)$ of convergence of the excess risk. Note that in the general case for an infinite class functions such an asymptotic rate with respect to the sample size $n$ is \textit{unimprovable} (for example, see \cite[Theorem 3 and Section 5]{rakhlin2017empirical}).
    \clearpage

{\section{Experiments}
\label{experiments}}

\subsection{Realization}
\SetKwComment{Comment}{/* }{ */}
\RestyleAlgo{ruled}

\begin{wrapfigure}[10]{r}{0.55\linewidth}
    \vspace{-10pt}
    \begin{algorithm}[H]
    \caption{$Star_d$ algorithm}\label{alg:two}
    \KwData{class of functions $\mathcal{F}$ and parameter $d$}
    \KwResult{estimator $\widehat{f}$}
    $\widehat{g}_1\gets calculate\_erm(set = \mathcal{F}, seed = seed_1)$\;
    \dots\\
    $\widehat{g}_d\gets calculate\_erm(set = \mathcal{F}, seed = seed_d)$\;
    $\widehat{f} \gets calculate\_erm(set = Star_d, seed = seed)$\;
    \textbf{return} $\widehat{f}$
    \label{alg:star_d}
    \end{algorithm}  
    \hspace{-20pt}
\end{wrapfigure}

The proposed $Star_d$ procedure can be represented by the following pseudocode (see Algorithm \ref{alg:star_d}). Since in practice we will not be able to give an infinite number of predictors, in fact the algorithm takes the architecture of the neural network as $\mathcal{F}$, without fixing its parameters. The $calculate\_erm$ procedure is some optimization process that reduces the empirical risk. Also, in practice, it is impossible to search for a global empirical risk minimizer in the space of neural networks, which is why we introduced the concept of  $\Delta$-minimizers. As follows from our results, the more accurate the optimization is at each step of the algorithm, the higher the accuracy guarantee of the final predictor $\widehat{f}$.

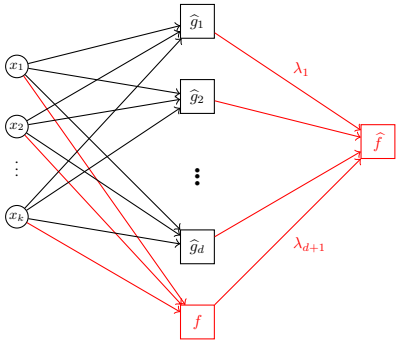
\begin{wrapfigure}[13]{r}{0.4\linewidth}
\begin{tikzpicture}[scale=0.8, every node/.style={scale=0.6}]
	\tikzstyle{unit}=[draw,shape=circle,inner sep=1pt,minimum size=0.5cm]
	\tikzstyle{ridden}=[draw=red!,shape=rectangle,inner sep=1pt,minimum size=0.75cm]
	\tikzstyle{hidden}=[draw, shape=rectangle,inner sep=1pt,minimum size=0.75cm]

	\node[unit](x0) at (0,4.75){\small $x_1$};
	\node[unit](x1) at (0,3.75){\small $x_2$};
	\node at (0,3.125){\vdots};
	\node[unit](xd) at (0,2.25){\small $x_k$};

	\node[hidden](h10) at (3,5.5){ $\widehat{g}_1$};
	\node[hidden](h11) at (3,4.25){ $\widehat{g}_2$};
	\node[font=\Huge] at (3,3){\vdots};
	\node[hidden](h1m) at (3,1.75){ $\widehat{g}_d$};
	\node[hidden,red](h1m1) at (3,0.5){$f$};

	\node[hidden, red](y) at (6,3.5){$\widehat{f}$};

	\draw[->] (x0) -- (h10);
	\draw[->] (x0) -- (h11);
	\draw[->] (x0) -- (h1m);
	\draw[->,red] (x0) -- (h1m1);

	\draw[->] (x1) -- (h10);
	\draw[->] (x1) -- (h11);
	\draw[->] (x1) -- (h1m);
	\draw[->,red] (x1) -- (h1m1);
	
	\draw[->] (xd) -- (h10);
	\draw[->] (xd) -- (h11);
	\draw[->] (xd) -- (h1m);
	\draw[->,red] (xd) -- (h1m1);
	
	\draw[->, red] (h10) -- node [text width=2cm,midway,above right] {$\lambda_1$}(y);
	\draw[->, red] (h11) -- node [text width=2cm,midway,below right] {$ $}(y);
	\draw[->, red] (h1m) -- node [text width=2cm,midway,above] {$ $}(y);
	\draw[->, red] (h1m1) -- node [text width=2cm,midway,below right] {$\lambda_{d+1}$}(y);
 	\end{tikzpicture}
 	
 	\caption{Illustrate $Star_d$ algorithm on NN}
 	\label{fig:illustrate_alg}
\end{wrapfigure}

    Despite the fact that the second step of the star algorithm requires an optimization procedure over some complex set $Star_d$, this is fairly easy to implement in practice (see Figure \ref{fig:illustrate_alg}). Suppose that we have fixed some architecture of estimator (black block), then in the first step we independently optimize the weights of the blocks $\widehat{g}_i$, freeze them and in the second step we add a new block $f$, connecting all of them by convex layer $\lambda_i$ (red elements) and optimize them. This actually iterates over all possible simplices, optimizing the weights of the lower block, and all possible points within the simplex, optimizing the convex weights. To ensure the convexity coefficients in practice we took softmax from the parameter vector.

\subsection{Competitors and other details}
\label{competitors}
    Looking at the structure of the proposed star procedure (let us call it \textbf{Classic Star (no warm-up)}), a desire naturally arises to compare it with the following models. First, with the model where all $d+1$ blocks together with the connecting \textit{linear} layer are trained simultaneously (\textbf{Big NN}), and second, where all $d+1$ blocks are trained independently, and then their predictions are averaged (\textbf{Ensemble}).
    
    But the star procedure has a number of possible problems. Training even one neural network is a rather complicated process, while in our algorithm it is required to train $d+1$ predictors. 
    To solve this problem, we propose to train the $d+1$ models in the ensemble method and the first $d$ models in the Classic Star procedure using the snapshot technique, calling such approaches \textbf{Snap Ensemble} and \textbf{Snap Star}, respectively.
    
    Additionally, at the last step of the star procedure, when optimizing the weights of the last model and convex coefficients, the minimizer has the opportunity to set the contribution of the last model to zero and focus on the aggregation of already trained $d$ models. We offer variations of the star algorithm that try to deal with these. 
    It is possible to train the last snapshot (model $d$) from the Snap Star as a $d + 1$ model, this is what we call \textbf{Snap Star (shot warm-up)}, also can add a new model to the Snap Star or Classic Star and spend part of the computational budget on training only its weights, the so-called warm-up, and get \textbf{Snap Star (new warm-up)} and \textbf{Classic Star (new warm-up)}, respectively. We compare all these methods with each other on a number of tasks.
    The amount of the computational budget for training a separate block we will denote as $epochs$, in the snapshot technique \textit{learning rate} decreases with cyclical cosine annealing \cite{loshchilov2016sgdr}.
    Architecture of neural networks and hardware for training are be reported for each task separately.
    We ran all experiments using the \fnurl{PyTorch}{https://pytorch.org/ library} library. For the purposes of reproducing the results, the code and extended tables with results are publicly available at \fnurl{repository}{https://github.com/mordiggian174/star-ensembling}.

\subsection{Datasets}
\label{datasets}

\textbf{Boston House Pricing}
    We conduct the first experiment on the Boston House Pricing dataset \cite{harrison1978hedonic}, whose task is to predict the value of real estate according to some characteristics. The ratio of training and test samples is equal to $7:3$. Standard scaler was used as preprocessing, batch size is $32$. A small fully connected ReLu neural network of $4$ layers was chosen as the architecture of the neural network, the number of neurons on the first layer is $128$, then with the growth of the layer it decreases by $2$ times, DropOut with parameter $p$ and batch normalization are applied between the layers. For warming up a new model, 40\% of the budget for block training is spent. Training was done using Adam \cite{kingma2014adam} on the CPU Apple M1. Error results on the test set were averaged over 5 runs.
    Some of the data obtained in the course of numerical experiments we present in the Table \ref{table:boston200}. Hereinafter, we highlight in \textbf{bold} the best results with a fixed parameter of the number of models $d$.
    
    It is worth adding that without the use of the DropOut technique, the usual Ensemble with some parameters proved to be quite good (e.g. \textbf{10.236} MSE after 100 $epochs$ with $d=5$, $p=0.0$, $lr=0.01$). 
    In total, about 100 different sets of parameters were tried on this launch, we attach part of the results in Appendix \ref{Appendix B}, and the full versions can be found in the repository. 
\begin{table}[ht!]
\centering
\begin{tabular}{|l|llllll|}
\hline
Name                    & d & MSE          & MAE   & $R^2$ & TRAIN MSE & TIME (sec) \\
\hline \hline
Snap Star (shot warm-up)   & 5 & 11.155±0.967 & 2.185  & 0.866     & 1.232     & 52.4 \\
Snap Star (new warm-up)    & 5 & 10.677±1.034 & 2.148  & 0.872     & 1.306     & 49.2 \\
Snap Ensemble           & 5 & 11.123±0.615 & 2.206  & 0.866     & 1.728     & 43.4 \\
Ensemble                & 5 & 10.172±0.621 & 2.164  & 0.878     & 2.133     & 43.2 \\
Classic Star (no warm-up)  & 5 & \textbf{9.650±0.911}  & 2.092  & \textbf{0.884}     & 1.373     & 52.0 \\
Classic Star (new warm-up) & 5 & 9.664±0.324  & \textbf{2.086} & \textbf{0.884}     & 1.374     & 45.8 \\
Big NN                  & 5 & 12.117±1.601 & 2.326  & 0.854     & 2.160     & 28.4 \\
\hline
Snap Star (shot warm-up)   & 4 & 10.745±0.640 & 2.154  & 0.871     & 1.192     & 43.0 \\
Snap Star (new warm-up)    & 4 & 10.862±0.537 & 2.162  & 0.870     & 1.401     & 43.0 \\
Snap Ensemble           & 4 & 11.169±0.558 & 2.205  & 0.866     & 1.679     & 36.0 \\
Ensemble                & 4 & 9.967±0.244  & 2.144  & 0.880     & 1.833     & 35.8 \\
Classic Star (no warm-up)  & 4 & 9.887±0.501  & \textbf{2.111}  & 0.881     & 1.326     & 43.6 \\
Classic Star (new warm-up) & 4 & \textbf{9.741±0.790}  & 2.114  & \textbf{0.883}     & 1.400     & 38.2 \\
Big NN                  & 4 & 11.856±0.429 & 2.323  & 0.858     & 1.937     & 24.0 \\
\hline
\end{tabular}
\caption{BOSTON HOUSE PRICING. Part of results at 200 epochs, $p=0.1$, $lr=0.01$}
\label{table:boston200}
\end{table}    
    
\textbf{Fashion Mnist}
    The second experiment was carried out on the Fashion Mnist dataset \cite{xiao2017fashion}, which consists of $70,000$ images ($28 \times 28$ pixels). It is required to classify images by clothing classes. The ratio of training and test samples is equal to $6:1$. No scaler is used, batch size is $64$. A simple convolutional network LeNet was chosen as a solution to this task. For warming up a new model, 40\% of $epochs$ is again spent.
    Training was done using Adam on the CPU Apple M1.
    Error results on the test set were averaged over 3 runs.

    As can be seen from the results (see Table \ref{table:fmnist15}), the algorithm proposed by us is again in the lead. 
    We also did one relatively hard run with parameters $d=5$ and 25 $epochs$ for train, the Classic Star (new warm-up) won with 92.2\% accuracy. At the moment, they take $10-11th$ place in the \fnurl{leaderboard}{https://paperswithcode.com/sota/image-classification-on-fashion-mnist} for this dataset. For comparison, the model from the $10th$ place has more than $380K$ parameters. It is worth noting that we used fairly simple models (about $266K$ parameters) and training took about $40$ minutes on a regular laptop. The results of training by our method also exceed the indicators
    \fnurl{Resnet18}{https://github.com/kefth/fashion-mnist} with over $2M$ parameters and 92\% accuracy. 
\begin{table}[ht!]
\centering
\scalebox{1}{
\begin{tabular}{|l|llll|}
\hline
Name                    & d & accuracy    & entropy     & TIME (sec)   \\
\hline
\hline
Snap Star (shot warm-up)   & 3 & 0.902±0.001 & 0.354±0.017 & 1021.3 \\
Snap Star (new warm-up)    & 3 & 0.902±0.001 & 0.348±0.015 & 940.3  \\
Snap Ensemble           & 3 & 0.904±0.002 & 0.274±0.006 & 819.7  \\
Ensemble                & 3 & 0.910±0.003 & 0.247±0.004 & 818.7  \\
Classic Star (no warm-up)  & 3 & \textbf{0.913±0.002} & 0.245±0.007 & 1014.7 \\
Classic Star (new warm-up) & 3 & \textbf{0.913±0.004} & \textbf{0.243±0.008} & 855.3  \\
Big NN                  & 3 & 0.905±0.004 & 0.297±0.011 & 643.0    \\
\hline
Snap Star (shot warm-up)   & 2 & 0.906±0.004 & 0.292±0.015 & 751.0    \\
Snap Star (new warm-up)    & 2 & 0.907±0.005 & 0.284±0.012 & 697.0    \\
Snap Ensemble           & 2 & 0.908±0.004 & 0.266±0.011 & 616.0    \\
Ensemble                & 2 & \textbf{0.912±0.004} & 0.251±0.004 & 614.3  \\
Classic Star (no warm-up)  & 2 & 0.910±0.000 & 0.245±0.003 & 747.3  \\
Classic Star (new warm-up) & 2 & 0.911±0.002 & \textbf{0.244±0.004} & 613.3  \\
Big NN                  & 2 & 0.902±0.001 & 0.303±0.007 & 498.3  \\
\hline
\end{tabular}
}
\caption{FASHION MNIST. Part of results at 15 epochs, $lr = 0.001$}
\label{table:fmnist15}
\end{table}    

\textbf{Million Song}
    For the last experiment we use UCI \cite{Dua:2019} subset of Million Song dataset\cite{Bertin-Mahieux2011}, containing 515345 songs, with 90 features each. Dataset's features are timbre average and covariance for every song, target is the year the song was released. Dataset is divided intro train and test in proportion $9 : 1$, where test is the last $10\%$ of the dataset. Features and targets were standard normalized prior to model training, batch size is 128. Base model for this task is fully-connected 4-layer ReLu neural network with layer dimensions $90, 120, 20, 1$. 
    After each hidden layer, batch normalization and DropOut with probability of $0.2$ are applied.
    Snapshots used SGD optimizer starting at \textit{learning rate} of $0.1$. All other models used Adam optimizer with \textit{learning rate} of $0.001$. In contrast to previous experiments, for warming up a new model, 10\% of $epochs$ is spent and shot warm-up used next fully trained snapshot instead of a last shot (that is, the total training budget is $(d+2)\cdot epochs$).
    Training was done on the NVIDIA GeForce GTX 1060 GPU (6gb) / Intel Core i5-7500. Error results were averaged over 5 runs.
    Some results are shown in the Table \ref{table:MSD30}.
\begin{table}
\centering
\begin{center}
\begin{tabular}{|l|llllll|} 
\hline
Name & d & MSE & MAE & $R^2$ & Train MSE & TIME (sec)\\
\hline
\hline
Snap Star (shot warm-up) & 5 & $76.70 \pm 0.28$ & 5.97 & 0.358 & 68.05 & 2369 \\
Snap Star (new warm-up) & 5 & $76.16 \pm 0.15$ & 5.96 & 0.363 & 69.39 & 2076 \\
Snap Ensemble & 5 & $76.48 \pm 0.18$ & 6.01 & 0.360 & 68.15 & 1766 \\
Ensemble & 5 & $75.74 \pm 0.06$ & 5.99 & 0.366 & 68.99 & 1955\\
Classic Star (no warm-up) & 5 & $75.97 \pm 0.20$ & 6.02 & 0.365 & 70.88 & 2150\\
Classic Star (new warm-up) & 5 & $\mathbf{75.62 \pm 0.06}$ & 5.96 & 0.367 & 70.29 & 2176\\
Big NN & 5 & $76.23 \pm 0.19$ & 5.95 & 0.362 & 71.70 & 1166\\
\hline
Snap Star (shot warm-up) & 4 & $76.74 \pm 0.20$ & 5.96 & 0.358 & 68.44 & 2026 \\
Snap Star (new warm-up) & 4 & $76.22 \pm 0.18$ & 5.97 & 0.362 & 69.74 & 1753 \\
Snap Ensemble & 4 & $76.48 \pm 0.19$ & 6.01 & 0.360 & 68.41 & 1473 \\
Ensemble & 4 & $75.78 \pm 0.07$ & 5.99 & 0.366 & 69.01 & 1636\\
Classic Star (no warm-up) & 4 & $76.02 \pm 0.09$ & 6.01 & 0.364 & 71.11 & 1811\\
Classic Star (new warm-up) & 4 & $\mathbf{75.76 \pm 0.17}$ & 5.96 & 0.366 & 70.41  & 1823\\
Big NN & 4 & $76.29 \pm 0.29$ & 5.96 & 0.362 & 71.74 & 980\\
\hline
\end{tabular}
\end{center}
\caption{MILLIION SONG. Part of results at 30 epochs}
\label{table:MSD30}
\end{table}    
    \subsection{General patterns}
    From the obtained results it follows that the best results are achieved by
    the proposed algorithm with the warm-up of the last model. We believe that the reason for this is
    that, with a weakly trained last model, it is easier for the algorithm to nullify the contribution of the last model and focus on the predictions of well-trained $d$ estimators.
    Although results of star algorithm with snapshot technique are mediocre for high number of $epochs$, it shows quite good results for smaller number of $epochs$
    (see Tables with a small number of $epochs$ in Appendix \ref{Appendix B}).

{\section{Discussion and limitations}
\label{discussion}}
    The proposed algorithm performs well in the classification problem with cross-entropy loss, although this paper only presents a theoretical analysis of regression with squared loss. The practical implementation of the algorithm can be sensitive to the initialization of the last block, as shown by the warm-up method. And experiments on larger models were not carried out due to high computational costs. Some solutions to these problems are offered in Section \ref{competitors}.

    In fact, the star estimator we proposed is a multidimensional analogue of the Audibert's algorithm. It combines optimal orders as a solution to the aggregation problem of model selection, and at the same time behaves like an ensemble method. This decision can be viewed from two sides at once.

\textit{$Star_d$ algorithm as a new learning algorithm for neural networks of block architecture}

     It is worth noting that if we spend a fixed amount of computing resources $B$ for each call to the optimization process $calculate\_erm$, then the total budget of our algorithm is about $(d + 1)\cdot B$.
     But the surprising fact is that the result obtained is able to compete with other methods for training the final large neural network from $d + 1$ blocks, although our theoretical analysis guarantees optimality only in comparison with the best single block architecture model. Thus, the procedure we proposed can be perceived as a new method for training neural networks with block architecture.
     
\textit{$Star_d$ algorithm as a new way of model aggregation}

    Also note that the predictors $\widehat{g}_i$ need not be trained in the first step. 
    Then the $Star_d$ algorithm can be perceived as an algorithm for aggregating these models. It will consist of the following: a new predictive model $f$ is added, a connecting layer, and the process of optimization by a parameter is started.
    At the same time, generally speaking, it is not necessary to have all blocks be of the same architecture.
    As intuition suggests, the main thing is that the expressive abilities of those classes of solutions to which the predictors given to us will relate should be approximately equal. Then it will be possible to formally consider the union of those decision classes to which each of the predictors belongs, and consider them  as $\Delta$-minimizers from the following class $\mathcal{F} =\bigcup\limits_i\mathcal{F}_i$, where given predictors $\widehat{g}_i \in \mathcal{F}_i$ (which may be heterogeneous).

\section{Acknowledgments}

We are grateful to Nikita Puchkin for essential comments and productive discussions, and also to Alexander Trushin for help with the design of the work. The article was prepared within the framework of the HSE University Basic Research Program. 



\printbibliography

@inproceedings{tsybakov03,
  title={Optimal Rates of Aggregation},
  author={Alexandre B. Tsybakov},
  booktitle={COLT},
  year={2003}
}

@article{lecue2009aggregation,
  title={Aggregation via empirical risk minimization},
  author={Lecu{\'e}, Guillaume and Mendelson, Shahar},
  journal={Probability theory and related fields},
  volume={145},
  number={3},
  pages={591--613},
  year={2009},
  publisher={Springer}
}

@article {a07-mixture,
	AUTHOR = {J.-Y. Audibert},
	TITLE = {Progressive mixture rules are deviation suboptimal},
    JOURNAL = { NeurIPS},
	YEAR = {2007},
}

@article{vijaykumar2021localization,
  title={Localization, Convexity, and Star Aggregation},
  author={Vijaykumar, Suhas},
  journal={Advances in Neural Information Processing Systems},
  volume={34},
  year={2021}
}

@inproceedings{liang2015learning,
  title={Learning with square loss: Localization through offset rademacher complexity},
  author={Liang, Tengyuan and Rakhlin, Alexander and Sridharan, Karthik},
  booktitle={Conference on Learning Theory},
  pages={1260--1285},
  year={2015},
  organization={PMLR}
}

@article{kawaguchi2016deep,
  title={Deep learning without poor local minima},
  author={Kawaguchi, Kenji},
  journal={Advances in neural information processing systems},
  volume={29},
  year={2016}
}

@inproceedings{caruana2004ensemble,
  title={Ensemble selection from libraries of models},
  author={Caruana, Rich and Niculescu-Mizil, Alexandru and Crew, Geoff and Ksikes, Alex},
  booktitle={Proceedings of the twenty-first international conference on Machine learning},
  pages={18},
  year={2004}
}

@article{schmidt2020nonparametric,
  title={Nonparametric regression using deep neural networks with ReLU activation function},
  author={Schmidt-Hieber, Johannes},
  journal={The Annals of Statistics},
  volume={48},
  number={4},
  pages={1875--1897},
  year={2020},
  publisher={Institute of Mathematical Statistics}
}

@article{rakhlin2017empirical,
  title={Empirical entropy, minimax regret and minimax risk},
  author={Rakhlin, Alexander and Sridharan, Karthik and Tsybakov, Alexandre B},
  journal={Bernoulli},
  volume={23},
  number={2},
  pages={789--824},
  year={2017},
  publisher={Bernoulli Society for Mathematical Statistics and Probability}
}

@article{drucker1994boosting,
  title={Boosting and other ensemble methods},
  author={Drucker, Harris and Cortes, Corinna and Jackel, Lawrence D and LeCun, Yann and Vapnik, Vladimir},
  journal={Neural Computation},
  volume={6},
  number={6},
  pages={1289--1301},
  year={1994},
  publisher={MIT Press}
}

@inproceedings{deng2009imagenet,
  title={Imagenet: A large-scale hierarchical image database},
  author={Deng, Jia and Dong, Wei and Socher, Richard and Li, Li-Jia and Li, Kai and Fei-Fei, Li},
  booktitle={2009 IEEE conference on computer vision and pattern recognition},
  pages={248--255},
  year={2009},
  organization={Ieee}
}

@article{nemirovski2000topics,
  title={Topics in non-parametric statistics},
  author={Nemirovski, Arkadi},
  journal={Lectures on probability theory and statistics (Saint-Flour, 1998)},
  volume={1738},
  pages={85--277},
  year={2000}
}

@inproceedings{qiu2014ensemble,
  title={Ensemble deep learning for regression and time series forecasting},
  author={Qiu, Xueheng and Zhang, Le and Ren, Ye and Suganthan, Ponnuthurai N and Amaratunga, Gehan},
  booktitle={2014 IEEE symposium on computational intelligence in ensemble learning (CIEL)},
  pages={1--6},
  year={2014},
  organization={IEEE}
}

@inproceedings{he2016deep,
  title={Deep residual learning for image recognition},
  author={He, Kaiming and Zhang, Xiangyu and Ren, Shaoqing and Sun, Jian},
  booktitle={Proceedings of the IEEE conference on computer vision and pattern recognition},
  pages={770--778},
  year={2016}
}

@article{otter2020survey,
  title={A survey of the usages of deep learning for natural language processing},
  author={Otter, Daniel W and Medina, Julian R and Kalita, Jugal K},
  journal={IEEE transactions on neural networks and learning systems},
  volume={32},
  number={2},
  pages={604--624},
  year={2020},
  publisher={IEEE}
}

@article{xiao2017fashion,
  title={Fashion-mnist: a novel image dataset for benchmarking machine learning algorithms},
  author={Xiao, Han and Rasul, Kashif and Vollgraf, Roland},
  journal={arXiv preprint arXiv:1708.07747},
  year={2017}
}

@article{harrison1978hedonic,
  title={Hedonic housing prices and the demand for clean air},
  author={Harrison Jr, David and Rubinfeld, Daniel L},
  journal={Journal of environmental economics and management},
  volume={5},
  number={1},
  pages={81--102},
  year={1978},
  publisher={Elsevier}
}

@article{kingma2014adam,
  title={Adam: A method for stochastic optimization},
  author={Kingma, Diederik P and Ba, Jimmy},
  journal={arXiv preprint arXiv:1412.6980},
  year={2014}
}

@inproceedings{kim2002support,
  title={Support vector machine ensemble with bagging},
  author={Kim, Hyun-Chul and Pang, Shaoning and Je, Hong-Mo and Kim, Daijin and Bang, Sung-Yang},
  booktitle={International workshop on support vector machines},
  pages={397--408},
  year={2002},
  organization={Springer}
}

@article{ha2005response,
  title={Response models based on bagging neural networks},
  author={Ha, Kyoungnam and Cho, Sungzoon and MacLachlan, Douglas},
  journal={Journal of Interactive Marketing},
  volume={19},
  number={1},
  pages={17--30},
  year={2005},
  publisher={Elsevier}
}

@inproceedings{liu2014facial,
  title={Facial expression recognition via a boosted deep belief network},
  author={Liu, Ping and Han, Shizhong and Meng, Zibo and Tong, Yan},
  booktitle={Proceedings of the IEEE conference on computer vision and pattern recognition},
  pages={1805--1812},
  year={2014}
}

@inproceedings{moghimi2016boosted,
  title={Boosted convolutional neural networks.},
  author={Moghimi, Mohammad and Belongie, Serge J and Saberian, Mohammad J and Yang, Jian and Vasconcelos, Nuno and Li, Li-Jia},
  booktitle={BMVC},
  volume={5},
  pages={6},
  year={2016}
}

@article{zhang2020snapshot,
  title={Snapshot boosting: a fast ensemble framework for deep neural networks},
  author={Zhang, Wentao and Jiang, Jiawei and Shao, Yingxia and Cui, Bin},
  journal={Science China Information Sciences},
  volume={63},
  number={1},
  pages={1--12},
  year={2020},
  publisher={Springer}
}

@inproceedings{wan2013regularization,
  title={Regularization of neural networks using dropconnect},
  author={Wan, Li and Zeiler, Matthew and Zhang, Sixin and Le Cun, Yann and Fergus, Rob},
  booktitle={International conference on machine learning},
  pages={1058--1066},
  year={2013},
  organization={PMLR}
}

@inproceedings{huang2016deep,
  title={Deep networks with stochastic depth},
  author={Huang, Gao and Sun, Yu and Liu, Zhuang and Sedra, Daniel and Weinberger, Kilian Q},
  booktitle={European conference on computer vision},
  pages={646--661},
  year={2016},
  organization={Springer}
}

@misc{snapshots,
  doi = {10.48550/ARXIV.1704.00109},
  url = {https://arxiv.org/abs/1704.00109},
  author = {Huang, Gao and Li, Yixuan and Pleiss, Geoff and Liu, Zhuang and Hopcroft, John E. and Weinberger, Kilian Q.},
  title = {Snapshot Ensembles: Train 1, get M for free},
  publisher = {arXiv},
  year = {2017},
  copyright = {arXiv.org perpetual, non-exclusive license}
}

@misc{FGE,
  doi = {10.48550/ARXIV.1802.10026},
  url = {https://arxiv.org/abs/1802.10026},
  author = {Garipov, Timur and Izmailov, Pavel and Podoprikhin, Dmitrii and Vetrov, Dmitry and Wilson, Andrew Gordon},
  title = {Loss Surfaces, Mode Connectivity, and Fast Ensembling of DNNs},
  publisher = {arXiv},
  year = {2018},
  copyright = {arXiv.org perpetual, non-exclusive license}
}

@misc{ensemble_review,
  doi = {10.48550/ARXIV.2104.02395},
  url = {https://arxiv.org/abs/2104.02395},
  author = {Ganaie, M. A. and Hu, Minghui and Malik, A. K. and Tanveer, M. and Suganthan, P. N.},
  title = {Ensemble deep learning: A review},
  publisher = {arXiv},
  year = {2021},
  copyright = {arXiv.org perpetual, non-exclusive license}
}

@article{dropout,
  author  = {Nitish Srivastava and Geoffrey Hinton and Alex Krizhevsky and Ilya Sutskever and Ruslan Salakhutdinov},
  title   = {Dropout: A Simple Way to Prevent Neural Networks from Overfitting},
  journal = {Journal of Machine Learning Research},
  year    = {2014},
  volume  = {15},
  number  = {56},
  pages   = {1929-1958},
  url     = {http://jmlr.org/papers/v15/srivastava14a.html}
}

@misc{Dua:2019 ,
author = "Dua, Dheeru and Graff, Casey",
year = "2017",
title = "{UCI} Machine Learning Repository",
url = "http://archive.ics.uci.edu/ml",
institution = "University of California, Irvine, School of Information and Computer Sciences" }

@article{loshchilov2016sgdr,
  title={Sgdr: Stochastic gradient descent with warm restarts},
  author={Loshchilov, Ilya and Hutter, Frank},
  journal={arXiv preprint arXiv:1608.03983},
  year={2016}
}

@INPROCEEDINGS{Bertin-Mahieux2011,
  author = {Thierry Bertin-Mahieux and Daniel P.W. Ellis and Brian Whitman and Paul Lamere},
  title = {The Million Song Dataset},
  booktitle = {{Proceedings of the 12th International Conference on Music Information
	Retrieval ({ISMIR} 2011)}},
  year = {2011}
}

\newpage
\appendix
\section{Proofs}\label{Appendix A}
    The combination of the following 2 Lemmas is a generalization of the geometric inequality proved by Liang et al. \cite{liang2015learning}. In many respects the scheme of the proof is similar.
    
    \begin{lemma}{(Geometric inequality for the exact $Star_d$ estimator in the second step)}\\
    \label{geom_ineq_lemma-1}
    Let $\hat{g}_1 \ldots \hat{g}_d$ be $\Delta_1$-empirical risk minimizers from the first step of the $Star_d$ procedure, $\wave{f}$ be the exact minimizer from the second step of the $Star_d$ procedure. Then, for $c_{\ref{geom_ineq_lemma-1}}=\frac{1}{18}$ the following inequality holds:
    \begin{equation}
    \label{geom_ineq_ineq-1}
        \HE (h - Y)^2 - \HE (\wave{f} - Y)^2  \ge c_{\ref{geom_ineq_lemma-1}} \HE (\wave{f} - h)^2 - 2\Delta_1.
    \end{equation}
    \end{lemma}
\begin{proof}
    For any function $f,\,g$ we denote the empirical $\ell_2$ distance to be $\|f\|_n := \left[\HE f^2\right]^\frac{1}{2}$, empirical product to be $\langle f, g \rangle_n := \HE{[fg]}$ and the square of the empirical distance between $\mathcal{F}$ and $Y$ as
    $r_1$. 
    By definition of $Star_d$ estimator for some $\lambda \in [0;1]$ we have:
    \begin{align*}
        \wave{f} = (1 - \lambda) \widehat{g} + \lambda f,
    \end{align*}    
    where $\widehat{g}$ lies in a convex hull of $\Delta_1$-empirical risk minimizers $\{\widehat{g}_i\}_{i=1}^d$.  
    Denote the balls centered at $Y$ to be $\mathcal{B}_1 := \mathcal{B}(Y, \sqrt{r_1})$, $\mathcal{B}_1' := \mathcal{B}(Y, \|\widehat{g}-Y\|_n)$ and $\mathcal{B}_2 := \mathcal{B}(Y, \|\wave{f} - Y\|_n)$. The corresponding spheres will be called $\mathcal{S}_1, \mathcal{S}_1', \mathcal{S}_2$. We have $\mathcal{B}_2 \subseteq \mathcal{B}_1$ and $\mathcal{B}_2 \subseteq \mathcal{B}_1'$.
    Denote by $\mathcal{C}$ the conic hull of $\mathcal{B}_2$ with origin $\widehat{g}$ and define the spherical cap outside the cone $\mathcal{C}$ to be $\mathcal{S} = \mathcal{S}'_1 \setminus \mathcal{C}$.
    
    First, $\wave{f} \in \mathcal{B}_2$ and it is a contact point of $\mathcal{C}$ and $\mathcal{S}_2$. Indeed, $\wave{f}$ is necessarily on a line segment between $\hat{g}$ and a point outside $\mathcal{B}_1$ that does not pass through the interior of $\mathcal{B}_2$ by optimality of $\wave{f}$. Let $K$ be the set of all contact points of $\mathcal{C}$ and $\mathcal{S}_2$ -- potential locations of $\wave{f}$.
    
    Second, for any $h \in \mathcal{F}$, we have $\|h - Y\|_n \ge \sqrt{r_1}$ i.e. any $h \in \mathcal{F}$ is not in the interior of $\mathcal{B}_1$. Furthermore, let $\mathcal{C}'$ be bounded subset cone $\mathcal{C}$ cut at $K$. Thus $h \in (int \mathcal{C})^c \cap (\mathcal{B}_1)^c$ or $h \in \mathcal{T}$, where $\mathcal{T} := (int\mathcal{C}') \cap (\mathcal{B}_1)^c$.
    
    For any  $h \in \mathcal{F}$ consider the two dimensional plane $\mathcal{L}$ that passes through three points $\hat{g},\, Y,\, h$, depicted in Figure \ref{fig:geom_ineq_picture}. Observe that the left-hand side of the desired inequality \eqref{geom_ineq_ineq-1} is constant as $\wave{f}$ ranges over $K$. The maximization of $\|h - f'\|^2_n$ over $f' \in K$ is achieved by $f' \in K \cap \mathcal{L}$. 
    Hence, to prove the desired inequality, we can restrict our attention to the plane $\mathcal{L}$ and $f'$.
    Let $h_\perp$ be the projection of $h$ onto the shell $L \cap S_1'$.
    By the geometry of the cone and triangle inequality we have:
    $$
        \|f' - \widehat{g}\|_n \ge \frac{1}{2}\|\widehat{g} - h_\perp\|_n \ge \frac{1}{2}\left(\|f' - h_\perp\|_n - \|f' - \widehat{g}\|_n\right),
    $$
    and, hence, $\|f' - \widehat{g}\|_n \geq \|f' - h_\perp\|_n / 3$.
    By the Pythagorean theorem,
    $$
        \|h_\perp - Y\|^2_n - \|f' - Y\|^2_n = \|\widehat{g} - Y\|^2_n - \|f' - Y\|^2_n = \|f' - \widehat{g}\|_n^2 \ge \frac{1}{9}\|f' - h_\perp\|^2_n.
    $$
    We can now extend this claim to $h$. Indeed, due to the geometry of the projection $h \to h_\perp$ and the fact that $h \in (int \mathcal{C})^c \cap (int\mathcal{B}_1)^c$ or $h \in \mathcal{T}$  there are 2 possibilities:
    
    a) $h \in (\mathcal{B}'_1)^c$. Then $\langle h_\perp - Y, h_\perp - h\rangle_n \le 0$;
    
    b) $h \in \mathcal{B}'_1$. Then, since $h \in (\mathcal{B}_1)^c$, we have
    $$
        \langle h_\perp - Y, h_\perp - h\rangle_n \leq \big(\|h-Y\|+\|h-h_\perp\|\big)\|h-h_\perp\| \le \|h_\perp - Y\|^2_n - \|h-Y\|^2_n \leq \Delta_1.
    $$
    In both cases, the following inequality is true
        \begin{align*}
            \|h - Y\|_n^2 - \|f' - Y\|_n^2
            &= \|h_\perp - h\|_n^2 - 2\langle h_\perp - Y, h_\perp - h\rangle_n + (\|h_\perp - Y\|_n^2 - \|f' - Y\|_n^2)\\
            &\ge \|h_\perp - h\|_n^2 - 2\Delta_1 + \frac{1}{9}\|f' - Y\|_n^2 \ge \frac{1}{18}\|f' - h\|_n^2 - 2\Delta_1.
        \end{align*}
\end{proof}

\begin{figure}[ht!]
    \centering
    \includegraphics[scale=1]{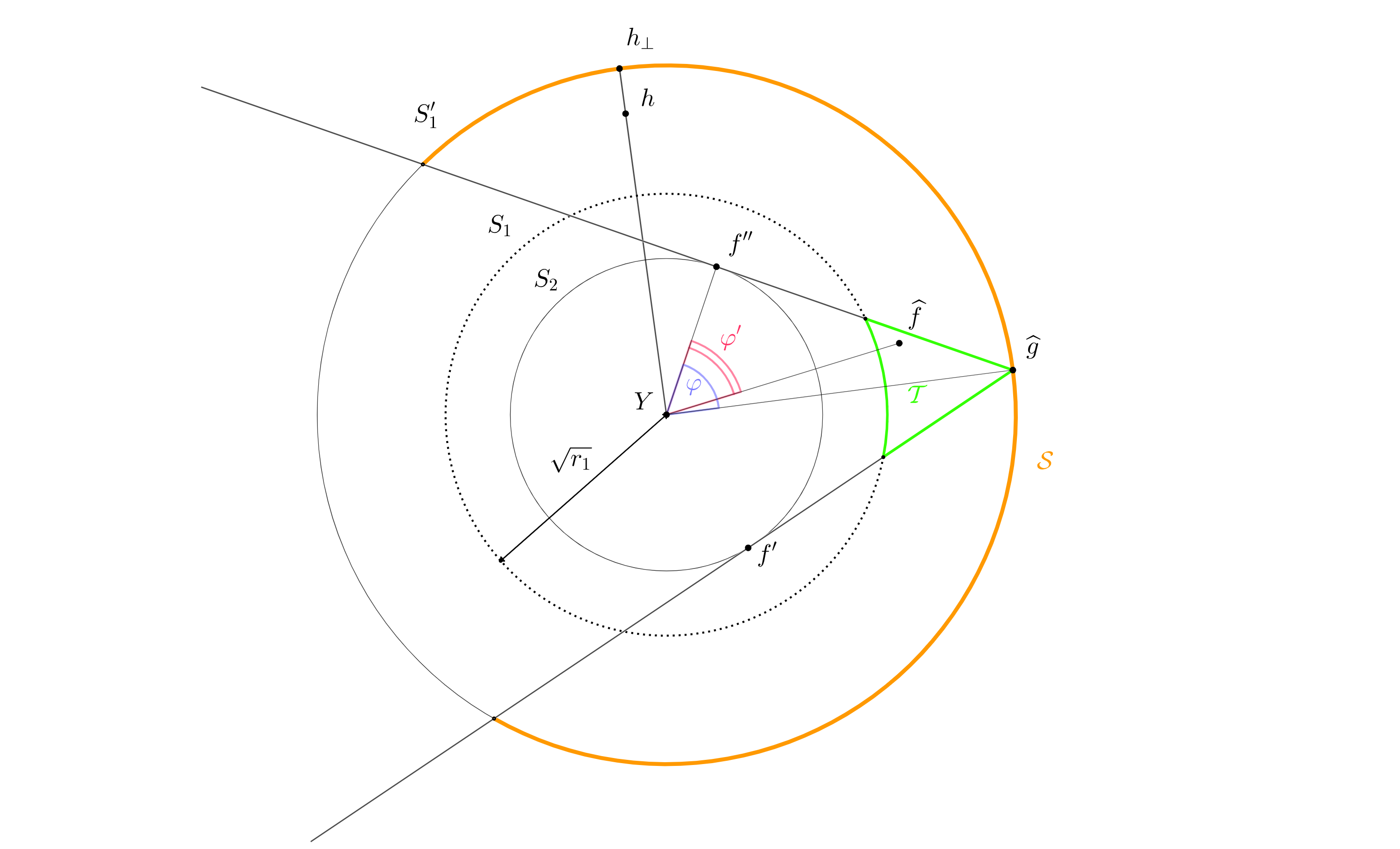}
    \caption{The cut surface $\mathcal{L}$}
    \label{fig:geom_ineq_picture}
\end{figure}

\begin{lemma}[Geometric Inequality for $\Delta$-empirical minimizers]
    \label{geom_ineq}
    Let $\hat{g}_1 \ldots \hat{g}_d$ be $\Delta_1$-empirical risk minimizers from the first step of the $Star_d$ procedure, and $\widehat{f}$ be the $\Delta_2$-empirical risk minimizer from the second step of the $Star_d$ procedure. Then, for any $h \in \mathcal{F}$ and $c_{\ref{geom_ineq}} = \frac{1}{36}$ the following inequality holds:
    \begin{equation*}
        \widehat{\E} (h - Y)^2 - \widehat{\E}(\widehat{f} - Y)^2  \ge c_{\ref{geom_ineq}} \widehat{E}(\widehat{f} - h)^2 - 2(1+c_{\ref{geom_ineq}})[\Delta_1 +\Delta_2].
    \end{equation*}
\end{lemma}
\begin{proof}
    Since Lemma \ref{geom_ineq_lemma-1} was actually proven for any $f \in K$, let $f''$ be the closest point to $\widehat{f}$ from $K$.      
    For this $f''$ the inequality \eqref{geom_ineq_ineq-1} holds. Similarly to Lemma \ref{geom_ineq_lemma-1}, there are 2 options: either $\widehat{f} \in (int\mathcal{C})^c$, or $\widehat{f} \in \mathcal{T}$.
    
    a) Let $\widehat{f} \in (int\mathcal{C})^c$, then $\langle \widehat{f} - f'', f''-Y \rangle \geq 0$.
    Since $\widehat{f}$ is $\Delta_2$-empirical risk minimizer, we have  $ \|\widehat{f} - f''\|^2_n + 2 \langle \widehat{f} - f'', f''-Y \rangle +\|f''-Y\|^2_n =\|\widehat{f}- Y \|^2_n  \leq \|f''-Y\|^2_n + \Delta_2$. It means, that $\|\widehat{f}-f''\|^2_n \leq \Delta_2$.
    
    b) Let $\widehat{f} \in \mathcal{T}$, then by the cosine theorem (as depicted on Figure \ref{fig:geom_ineq_picture},  $\mathcal{L}$ is the two dimensional plane which passes through $\widehat{f},\,\widehat{g},\,Y$): $$
    \| \widehat{f} - f''\|^2_n = \|f''- Y \|^2_n + \| \widehat{f} - Y\|^2_n - 2 \|f''-Y\|_n \| \widehat{f}-Y\|_n \cos(\varphi').$$
    But $\cos(\varphi') \geq \cos(\varphi) = \frac{\|f''-Y\|_n}{\|\widehat{g}-Y \|_n}$ and $\|\widehat{f}-Y\|^2_n \geq r_1$.
    Then we have:
    \begin{align*}
        \|\widehat{f}-f''\|^2_n &\leq \Delta_2 + 2 \| f''- Y\|^2_n \left(1-\frac{\| \widehat{f}-Y\|_n}{\|\widehat{g}-Y \|_n}\right) \\
        &\leq \Delta_2 + 2\frac{ \| f''- Y\|^2_n}{\|\widehat{g}-Y \|_n}\left(\frac{\|\widehat{g}-Y \|^2_n -  \| \widehat{f}-Y\|^2_n}{\|\widehat{g}-Y \|_n+\| \widehat{f}-Y\|_n}\right) \leq \Delta_1 + \Delta_2.
    \end{align*}
    
    Lemma \ref{geom_ineq_lemma-1} states:
    $$
        \|h - Y\|_n^2 \ge \|f'' - Y\|_n^2 + c_{\ref{geom_ineq_lemma-1}} \|f'' - h\|_n^2 - 2\Delta_1.
    $$
    By using the triangle inequality and the convexity of the quadratic function,
    we can get the following bound
    \begin{align*}
        \frac{c_{\ref{geom_ineq_lemma-1}}}{2} \|\widehat{f} - h\|^2_n  \le c_{\ref{geom_ineq_lemma-1}}\left(\|\widehat{f} - f''\|_n^2 + \|f'' - h\|_n^2\right) \le  c_{\ref{geom_ineq_lemma-1}}[\Delta_2+\Delta_1]+ c_{\ref{geom_ineq_lemma-1}}\|f'' - h\|_n^2.
    \end{align*}
    Combining everything together,  we get the required result for the constant $c_{\ref{geom_ineq}}=\frac{c_{\ref{geom_ineq_lemma-1}}}{2}=\frac{1}{36}$:
    $$
        \widehat{\E} (h - Y)^2 - \widehat{\E}(\widehat{f} - Y)^2  \ge c_{\ref{geom_ineq}} \cdot \widehat{\E}(\widehat{f} - h)^2 - 2(1+c_{\ref{geom_ineq}})[\Delta_1 + \Delta_2].
    $$
\end{proof}

    For convenience, we introduce a $\Delta$-excess risk
    $$ 
    \mathcal{E}_{\Delta}(\widehat{g}) := \E(\widehat{g} - Y)^2 - 
    \inf_{f \in \mathcal{F}}\E(f - Y)^2 - 2(1+c_{\ref{geom_ineq}})[\Delta_1 + \Delta_2],
    $$
    then the following 2 statements are the direct consequences of the corresponding statements from the article \cite{liang2015learning}. The only difference is that in our case the geometric inequality has terms on the right side with minimization errors $\Delta_1, \Delta_2$. Also our definition of the set $\mathcal{H}$ is different, but all that was needed from it was the property that $\widehat{f}$ lies in $\mathcal{H}+f^*$. For brevity, we will not repeat the proofs, but only indicate the numbers of the corresponding results in the titles of the assertions. We will also proceed for statements the proofs for which we slightly modify or use without changes.

        \begin{corollary}[Corollary 3]
        Conditioned on the data $\{({\textbf{X}_i}, Y_i) : 1 \leq i \leq n\}$, we have a deterministic upper bound for the $Star_d$ estimator:
        $$
                \mathcal{E}_{\Delta}(\widehat{f}) \le (\widehat \E - \E) [2(f^* - Y)(f^* - \widehat{f})] + \E (f^* -     \widehat{f})^2 - (1+c_{\ref{geom_ineq}}) \cdot \widehat{\E} (f^* - \widehat{f})^2.
        $$

    \end{corollary}
    \begin{theorem}[Theorem 4]
    \label{loss_expectation}
        The following expectation bound on excess loss of the $Star_d$ estimator holds:
        $$
            \E \mathcal{E}_{\Delta}(\widehat{f}) \le (2F' + F(2+c_{\ref{geom_ineq}})/2) \cdot {\E}_{\sigma} \sup_{h \in \mathcal{H}} \left\{\frac{1}{n}\sum_{i=1}^{n} 2\sigma_i h({\textbf{X}_i}) - c_{\ref{loss_expectation}}h({\textbf{X}_i})^2\right\},
        $$
        where $\sigma_1, \ldots \sigma_n$ are independent Rademacher random variables, $c_{\ref{loss_expectation}} = \min\left\{\frac{c_{\ref{geom_ineq}}}{4F'}, \frac{c_{\ref{geom_ineq}}}{4F(2+c_{\ref{geom_ineq}})}\right\}$, $F = \sup_{f \in \mathcal{F}} |f|_\infty$ and $F' = \sup_\mathcal{F} |Y - f|_\infty$ almost surely.
    \end{theorem}
    
    \begin{theorem}[Theorem 7]
    \label{loss_probability}
        Assume the lower isometry bound in Definition \ref{isom_bound} holds with $\eta_{lib} = c_{\ref{geom_ineq}}/4$ and some $\delta_{lib} < 1$ and $\mathcal{H}$ is the set defined in \ref{set_H}. Let $\xi_i = Y_i - f^*({\textbf{X}_i})$.
        Define
        $$
            A := \sup_{h \in \mathcal{H}} \frac{\E h^4}{(\E h^2)^2} \text{ and } B := \sup_{\textbf{X}, Y}\E\xi^4.
        $$
        Then there exist two absolute constants $c_{\ref{loss_probability}}' , \tilde{c_{\ref{loss_probability}}} > 0$ (which only depend on $c_{\ref{geom_ineq}}$), such that
        $$
            \Prob\left(\mathcal{E}_{\Delta}(\widehat{f}) > 4u\right) \le 4 \delta_{lib} + 4 \Prob\left(\sup_{h \in \mathcal{H}} \frac{1}{n} \sum_{i = 1}^{n} \sigma_i \xi_i h({\textbf{X}_i}) - \tilde{c_{\ref{loss_probability}}} h({\textbf{X}_i})^2 > u\right)
        $$
        for any
        $$
            u > \frac{32\sqrt{AB}}{c_{\ref{loss_probability}}'} \frac{1}{n}
        $$
        as long as $n > \frac{16(1-c_{\ref{loss_probability}}')^2 A}{c_{\ref{loss_probability}}'^2} \lor n_0(\mathcal{H}, \delta_{lib}, c_{\ref{geom_ineq}}/4)$.\\
    \end{theorem}

\begin{lemma}[Lemma 15]
\label{bound_complexity}
    The offset Rademacher complexity for $\mathcal{H}$ is bounded as:
    $$
        \mathbb{E}_{\sigma}\sup_\mathcal{H}\left\{\frac{1}{n}\sum_{i=1}^n 2\sigma_i \xi_i h({\textbf{X}_i}) - Ch({\textbf{X}_i})^2\right\} \le K(C)\varepsilon + M(C)\cdot \frac{\log \mathcal{N}_2(\mathcal{H}, \varepsilon)}{n}
    $$
    and with probability at least $1 - \delta$
    $$
        \sup_\mathcal{H} \left\{\frac{1}{n}\sum_{i=1}^n 2\sigma_i \xi_i h({\textbf{X}_i}) - Ch({\textbf{X}_i})^2\right\} \le K(C)\varepsilon + M(C) \cdot \frac{\log \mathcal{N}_2(\mathcal{H}, \varepsilon) + \log \frac{1}{\delta}}{n},
    $$
    where  
    \begin{equation}
    \label{K,M}
        K(C):=2\left(\sqrt{\sum_{i=1}^n {\xi^2/n}}+C\right), \,\,\,
        M(C):= \sup_{h \in \mathcal{H} \setminus \{0\}}4\frac{\sum_{i=1}^{n}h({\textbf{X}_i})^2\xi_i^2}{C \sum_{i=1}^{n}h({\textbf{X}_i})^2}.
    \end{equation}
\end{lemma}
\begin{proof}
        Let $N_2({\mathcal{H},\varepsilon})$ be the  $\varepsilon$-net of the $\mathcal{H}$ of size at most $\mathcal{N}_2({\mathcal{H},\varepsilon})$
        and $v[h]$ be the closest point from this net for function $h \in \mathcal{H}$, i.e. $\|h - v[h]\|_2 \le \varepsilon$.
        By using the inequality
        $
            v[h]_i^2 \le 2\Big(h_i^2 + (v[h]_i - h_i)^2\Big),
        $
        we can get next upper bound:
        \begin{multline*}
             \left\{\frac{1}{n}\sum_{i=1}^n 2\sigma_i \xi_i h({\textbf{X}_i}) - Ch({\textbf{X}_i})^2\right\}\\
              \le\left\{\frac{1}{n}\sum_{i=1}^n 2\sigma_i \xi_i (h({\textbf{X}_i}) - v[h]({\textbf{X}_i})) + C\Big(v[h]^2({\textbf{X}_i})/2 - h^2({\textbf{X}_i})\Big)\right\}\\
              + \frac{1}{n}\sup_{v \in {N}_2(\mathcal{H}, \varepsilon)}\left\{\sum_{i = 1}^n 2 \sigma_i \xi_i v({\textbf{X}_i}) - \frac{C}{2}v({\textbf{X}_i})^2\right\}\\
              \le2\varepsilon\left(\sqrt{\sum_{i=1}^n \xi_i^2/n} + C\right) + \frac{1}{n}\sup_{v \in {N}_2(\mathcal{H}, \varepsilon)}\left\{\sum_{i = 1}^n 2 \sigma_i \xi_i v({\textbf{X}_i}) - \frac{C}{2}v({\textbf{X}_i})^2\right\}.
        \end{multline*}
    The right summarand is supremum over set of cardinality not more than $\mathcal{N}_2(\mathcal{H},\varepsilon)$.
    By using Lemma \ref{finite_bound}, we acquire the expected estimates.
\end{proof}

    We have now obtained, using the offset Rademacher complexity technique, the upper bound on
    excess risk in terms of the coverage size of the set $\mathcal{H}$.
    To get the desired result,  
    we need to obtain an upper bound on the size of the cover $\mathcal{H}$ in terms of the size of the cover $\mathcal{F}$.
    \begin{lemma}
    \label{cover_ineq}
        For any scale $\varepsilon > 0$, the covering number of $\mathcal{F} \subseteq V(L+1)\cdot\mathcal{B}_2$ (where $\mathcal{B}_2$ is a sphere of radius one in space with norm $\|\cdot\|_n$) and that of $\mathcal{H}$ are bounded in the sense:
        $$
            \log \mathcal{N}_2(\mathcal{F},\varepsilon) \le \log \mathcal{N}_2(\mathcal{H}, \varepsilon) \le (d+2)\left[\log \mathcal{N}_2\left(\mathcal{F}, \frac{\varepsilon}{3(d+1)}\right) +\log\frac{6(d+1)V(L+1)}{\varepsilon}\right].
        $$
    \end{lemma}
\begin{proof}
        If we define as $N(\mathcal{F},\,\varepsilon)$ the $\varepsilon$-net
        cardinality no more then $\mathcal{N}(\mathcal{F},\,\varepsilon)$,
        then the following is true: $N({\mathcal{F}_1,\,\varepsilon_1}) + N({\mathcal{F}_1,\,\varepsilon_2})$ is $(\varepsilon_1 + \varepsilon_2)$-net for  $\mathcal{F}_1 + \mathcal{F}_2$. Hence,
        $
            \label{cover_sum}
            \mathcal{N}\left(\mathcal{F}_1 + \mathcal{F}_2, \,\varepsilon_1 + \varepsilon_2\right) \le \mathcal{N}\left(\mathcal{F}_1, \varepsilon_1\right) \cdot \mathcal{N}\left(\mathcal{F}_2, \varepsilon_2\right).
        $
    With this we can obtain the following upper bound
    $$\mathcal{N}_2(\mathcal{H}, \varepsilon) \le \mathcal{N}_2(\mathcal{F}+Hull_d, \varepsilon)\leq \mathcal{N}_2\left(\mathcal{F},\frac{\varepsilon}{3}\right) \cdot \mathcal{N}_2\left(Hull_d, \frac{2\varepsilon}{3}\right).$$
    
    But since $Hull_d$ is the sum of $d+1$ functions from $\mathcal{F}$ with coefficients in $[-1;1]$, by the inequaility \eqref{shmidt_norm}, we can cover this with a net of size no more than
    $$ \left[\mathcal{N}_2\left(\mathcal{F},\frac{\varepsilon}{3(d+1)}\right)\cdot\frac{6(d+1)V(L+1)}{\varepsilon}\right]^{d+1}.$$
\end{proof}

    Note that to obtain the required orders, we only need coverage with $\varepsilon = 1/n$.
    
\begin{corollary}
\label{bound_H}
Let $\mathcal{H}$ defined in \ref{set_H} for $\mathcal{F} = \mathcal{F}\left(L,\,\textbf{p},\,s\right)$, then for $V$ defined in \ref{V} holds
$$\log\mathcal{N}_2\left(\mathcal{H}, \frac{1}{n}\right) \leq c_{\ref{bound_H}}  d \, s \log \big(V L \,n \,d\big),$$
where $c_{\ref{bound_H}}$ is an indepedent constant.
\end{corollary}
\begin{proof}
By lemma \ref{cover_ineq} and inequality \ref{shmidt_cover}, we have
\begin{align*}
\log \mathcal{N}_2(\mathcal{H}, 1/n) &\le (d+2)\left[\log \mathcal{N}_2\left(\mathcal{F}(L,\,\textbf{p},\,s), \frac{1}{3n(d+1)}\right)     +\log 6n(d+1)V(L+1)\right]\\
&\leq (d+2)\left[(s+1)\log\left(2V^2(L+1)(3n(d+1))\right)  +\log\left(6n(d+1)V(L+1)\right)\right].
\end{align*}

\end{proof}
    We are now fully prepared to prove the two main results.
    
        \begin{theorem}
        \label{premain_th}
        Let $\widehat{f}$ be a $Star_d$ estimator and $\mathcal{H}$ be the set defined in \ref{set_H} for $\mathcal{F} = \mathcal{F}(L,\,\textbf{p},\,s)$. The following expectation bound on excess loss holds:
        $$
            \E \mathcal{E}_{\Delta}(\hat{f}) \le 2(F'+V(L+1)) \cdot \left[ \frac{K(C)}{n} + M(C) \cdot \frac{c_{\ref{bound_H}}  d \, s \log \left(V L \,n \,d\right)}{n}\right],
        $$
        where $K(C),\,\,M(C)$ defined in \eqref{K,M} for constants 
        $$ C = \min\left\{\frac{c_{\ref{geom_ineq}}}{4F'}, \frac{c_{\ref{geom_ineq}}}{4V(L+1)(2+c_{\ref{geom_ineq}})}\right\}, \,\,\,
        F' = \sup_\mathcal{F} |Y - f|_\infty.
        $$
        
    
    \end{theorem}
\begin{proof}
    By using Theorem \ref{loss_expectation} and inequality \ref{shmidt_norm} we have
        $$
            \E \mathcal{E}_{\Delta}(\hat{f}) \le (2F' + V(L+1)(2+c_{\ref{geom_ineq}})/2) \cdot {\E}_\sigma \sup_{h \in \mathcal{H}} \left\{\frac{1}{n}\sum_{i=1}^{n} 2\sigma_i h({\textbf{X}_i}) - Ch({\textbf{X}_i})^2\right\}, 
        $$
        where $C = \min\left\{\frac{c_{\ref{geom_ineq}}}{4F'}, \frac{c_{\ref{geom_ineq}}}{4V(L+1)(2+c_{\ref{geom_ineq}})}\right\}$, $F' = \sup_\mathcal{F} |Y - f|_\infty$ almost surely.

By using  Lemma \ref{bound_complexity} and corollary \ref{bound_H} we get desired result
    $$
        \mathbb{E}_{\sigma}\sup_\mathcal{H}\left\{\frac{1}{n}\sum_{i=1}^n 2\sigma_i \xi_i h({\textbf{X}_i}) - Ch({\textbf{X}_i})^2\right\} \le \frac{K(C)}{n} + M(C) \cdot \frac{c_{\ref{bound_H}}  d \, s \log \big(V L \,n \,d\big)}{n}.
    $$
\end{proof}

\begin{theorem}
    \label{main_th}
        Let $\widehat{f}$ be a $Star_d$ estimator and let $\mathcal{H}$ be the set defined in \ref{set_H} for $\mathcal{F} = \mathcal{F}(L,\,\textbf{p},\,s)$.
        Assume for $\mathcal{H}$ the lower isometry bound in Definition \ref{isom_bound} holds with $\eta_{lib} = c_{\ref{geom_ineq}}/4$ and some $\delta_{lib} < 1$. Let $\xi_i = Y_i - f^*({\textbf{X}_i})$.
        Define
        $$
            A := \sup_{h \in \mathcal{H}} \frac{\E h^4}{(\E h^2)^2} \,\text{ and } B := \sup_{\textbf{X}, Y}\E\xi^4.
        $$
        Then there exist 3 absolute constants $c_{\ref{main_th}}' , \tilde{c_{\ref{main_th}}}, c_{\ref{main_th}} > 0$ (which only depend on $c_{\ref{geom_ineq}}$), such that 
        $$
            \Prob\left(\mathcal{E}_{\Delta}(\widehat{f}) > 4D\right)\leq 4(\delta_{lib}+\delta)
        $$
        as long as $n > \frac{16(1-c_{\ref{main_th}}')^2 A}{c_{\ref{main_th}}'^2} \lor n_0(\mathcal{H}, \delta_{lib}, c_{\ref{main_th}}/4)$, 
        where 
        $$
            K:=\left(\sqrt{\sum_{i=1}^n {\xi^2/n}}+2\tilde{c_{\ref{main_th}}}\right), \,\,\,
            M:= \sup_{h \in \mathcal{H} \setminus \{0\}}\frac{\sum_{i=1}^{n}h({\textbf{X}_i})^2\xi_i^2}{\tilde{c_{\ref{main_th}}} \sum_{i=1}^{n}h({\textbf{X}_i})^2},
        $$
        $$
            D:= \max\left(\frac{K}{n} + M \cdot \frac{c_{\ref{bound_H}}  d \, s \log \big(V L \,n \,d\big) + \log \frac{1}{\delta}}{n}, \frac{32\sqrt{AB}}{c_{\ref{main_th}}'} \frac{1}{n}\right) 
        $$
        and $c_{\ref{bound_H}}$ is an independent constant.
\end{theorem}

\begin{proof}

    By using Theorem \ref{loss_probability} for any
        $
            u > \frac{32\sqrt{AB}}{c_{\ref{loss_probability}}'} \frac{1}{n}
        $ we have
        $$
            \Prob\left(\mathcal{E}_{\Delta}(\widehat{f}) > 4u\right) \le 4 \delta_{lib} + 4 \Prob\left(\sup_{h \in \mathcal{H}} \frac{1}{n} \sum_{i = 1}^{n} \sigma_i \xi_i h({\textbf{X}_i}) - \tilde{c_{\ref{loss_probability}}} h({\textbf{X}_i})^2 > u\right)
        $$

        as long as $n > \frac{16(1-c_{\ref{loss_probability}}')^2 A}{c_{\ref{loss_probability}}'^2} \lor n_0(\mathcal{H}, \delta_{lib}, c_{\ref{geom_ineq}}/4)$.
        
By using Lemmas \ref{bound_complexity} and \ref{bound_H} we have with probability no more than $\delta$ for any $C>0:$

    $$
        \sup_\mathcal{H} \left\{\frac{1}{n}\sum_{i=1}^n \sigma_i \xi_i h({\textbf{X}_i}) - \frac{C}{2}h({\textbf{X}_i})^2\right\} \ge \frac{K(C)}{2}\varepsilon + \frac{M(C)}{2} \cdot \frac{\log \mathcal{N}_2(\mathcal{H}, \varepsilon) + \log \frac{1}{\delta}}{n},
    $$
    where $K(C),\,\,M(C)$ are defined in \eqref{K,M}.
    Combining this inequality for $C = 2\tilde{c_{\ref{main_th}}}= 2\tilde{c_{\ref{loss_probability}}} $ and $c_{\ref{main_th}}'=c_{\ref{loss_probability}}'$, $c_{\ref{main_th}}=c_{\ref{geom_ineq}}$ we get the required result.
\end{proof}

    \begin{lemma}[Lemma 9]
    \label{finite_bound}
        Let $V \subset \mathbb{R}^n$ be a finite set, $|V| = N$. Then, for any $C>0:$
        $$
            {\E}_{\sigma}\max_{v \in V}\left[\frac{1}{n}\sum_{i=1}^n\sigma_i \xi_iv({\textbf{X}_i}) - Cv({\textbf{X}_i})^2\right] \le M\frac{\log N}{n}.
        $$
        For any $\delta > 0$:
        $$
            \mathbb{P}_{\sigma}\left(\max_{v \in V}\left[\frac{1}{n}\sum_{i=1}^n\sigma_i \xi_i v({\textbf{X}_i}) - Cv({\textbf{X}_i})^2\right] > M\frac{\log N + \log \frac{1}{\delta}}{n}\right) \le \delta,
        $$
        where
        $$
            M := \sup_{v \in V \setminus \{0\}} \frac{\sum_{i =1}^n v({\textbf{X}_i})^2 \xi_i^2}{2C\sum_{i =1}^n v({\textbf{X}_i})^2}.
        $$
    \end{lemma}

\clearpage    
\section{Result Tables}\label{Appendix B}

Here we additionally present tables with the results of numerical experiments. Particularly for runs with a small number of $epochs$. It can be observed that the SnapStar algorithm is quite good with a strong budget constraint.
The results also include a relatively large run for the FASHION MNIST dataset. At the moment, ClassicStar (new warm-up) takes $10-11th$ place in the \fnurl{leaderboard}{https://paperswithcode.com/sota/image-classification-on-fashion-mnist} for this dataset.
Full versions of the following tables can be found in the \fnurl{repository}{https://github.com/mordiggian174/star-ensembling}.

\begin{table}[ht!]
\centering
\scalebox{1}{
\begin{tabular}{|l|llllll|}
\hline
Name                    & d & MSE          & MAE    & $R^2$ & TRAIN MSE & TIME (sec) \\
\hline
\hline
Snap Star (shot warm-up)   & 5 & \textbf{10.881±0.575} & \textbf{2.229} & \textbf{0.869}     & 1.976     & 7.8  \\
Snap Star (new warm-up)    & 5 & 11.285±0.650 & 2.283  & 0.864     & 2.656     & 6.6  \\
Snap Ensemble           & 5 & 11.862±0.616 & 2.306   & 0.858     & 2.629     & 6.6  \\
Ensemble                & 5 & 12.568±0.878 & 2.399   & 0.849     & 4.220     & 6.8  \\
Classic Star (no warm-up)  & 5 & 11.365±0.410 & 2.278   & 0.864     & 2.978     & 7.2  \\
Classic Star (new warm-up) & 5 & 12.157±0.822 & 2.353   & 0.854     & 3.320     & 6.2  \\
Big NN                  & 5 & 12.068±0.860 & 2.411   & 0.855     & 3.644     & 4.0  \\
\hline
Snap Star (shot warm-up)   & 4 & \textbf{11.276±0.582} & \textbf{2.269} & \textbf{0.865}     & 2.329     & 6.2  \\
Snap Star (new warm-up)    & 4 & 11.598±0.729 & 2.292   & 0.861     & 2.739     & 5.0  \\
Snap Ensemble           & 4 & 11.819±0.341 & 2.316   & 0.858     & 2.819     & 5.0  \\
Ensemble                & 4 & 12.059±0.614 & 2.365  & 0.855     & 3.732     & 5.0  \\
Classic Star (no warm-up)  & 4 & 11.608±0.722 & 2.286  & 0.861     & 3.198     & 6.2  \\
Classic Star (new warm-up) & 4 & 11.890±0.966 & 2.319   & 0.857     & 3.093     & 5.2  \\
Big NN                  & 4 & 12.556±0.904 & 2.383   & 0.849     & 3.746     & 4.0  \\
\hline
\end{tabular}
}
\caption{BOSTON HOUSE PRICING. Part of results at 30 epochs, $p=0.1$, $lr = 0.01$}
\label{table:boston30}
\end{table}
\begin{table}[ht!]
\centering
\begin{center}
\begin{tabular}{ |l|llllll|}
\hline
Name & d & MSE & MAE & R2 & TRAIN MSE & TIME (sec)\\
\hline
\hline
Snap Star (shot warm-up) & 5 & $76.31 \pm 0.17$ & 5.97 & 0.362 & 70.64 & 733 \\
Snap Star (new warm-up) & 5 & $76.21 \pm 0.10$ & 5.99 & 0.363 & 71.34 & 667 \\
Snap Ensemble & 5 & $76.42 \pm 0.11$ & 6.02 & 0.361 & 70.03 & 543 \\
Ensemble & 5 & $76.34 \pm 0.07$ & 6.05 & 0.361 & 72.05 & 711 \\
Classic Star (no warm-up) & 5 & $76.57 \pm 0.15$ & 6.07 & 0.36 & 73.62 & 783 \\
Classic Star (new warm-up) & 5 & $\mathbf{76.06 \pm 0.10}$ & 6.00 & 0.364 & 72.59 & 807 \\
Big NN & 5 & $77.04 \pm 0.21$ & 6.02 & 0.356 & 75.62 & 436 \\
\hline
Snap Star (shot warm-up) & 4 & $76.30 \pm 0.12$ & 5.99 & 0.362 & 71.04 & 632 \\
Snap Star (new warm-up) & 4 & $76.14 \pm 0.11$ & 6.01 & 0.363 & 71.78 & 565 \\
Snap Ensemble & 4 & $76.46 \pm 0.12$ & 6.02 & 0.360 & 70.37 & 452 \\
Ensemble & 4 & $76.40 \pm 0.08$ & 6.05 & 0.361 & 72.08 & 593 \\
Classic Star (no warm-up) & 4 & $76.51 \pm 0.04$ & 6.04 & 0.36 & 73.76 & 652 \\
Classic Star (new warm-up) & 4 & $\mathbf{76.01 \pm 0.10}$ & 6.01 & 0.364 & 72.69 & 676 \\
Big NN & 4 & $77.06 \pm 0.18$ & 6.03 & 0.355 & 75.63 & 375 \\
\hline
Snap Star (shot warm-up) & 3 & $76.39 \pm 0.32$ & 5.98 & 0.361 & 71.62 & 530 \\
Snap Star (new warm-up) & 3 & $\mathbf{76.10 \pm 0.07}$ & 6.00 & 0.363 & 72.38 & 463 \\
Snap Ensemble & 3 & $76.53 \pm 0.14$ & 6.02 & 0.360 & 70.77 & 362 \\
Ensemble & 3 & $76.43 \pm 0.09$ & 6.05 & 0.361 & 72.12 & 473 \\
Classic Star (no warm-up) & 3 & $76.51 \pm 0.12$ & 6.04 & 0.360 & 74.00 & 522 \\
Classic Star (new warm-up) & 3 & $76.16 \pm 0.13$ & 6.01 & 0.363 & 72.81 & 546 \\
Big NN & 3 & $76.80 \pm 0.23$ & 6.03 & 0.358 & 75.60 & 315 \\
\hline
\end{tabular}
\end{center}
\caption{MILLIION SONG. Part of results at 10 epochs}
\label{table:MSD10}
\end{table}
\begin{table}[ht!]
\centering
\scalebox{1}{
\begin{tabular}{|l|llll|}
\hline
Name                    & d & accuracy    & entropy     & TIME (sec)    \\
\hline
\hline
Snap Star (shot warm-up)   & 3 & \textbf{0.900±0.002} & \textbf{0.284±0.008} & 340.333 \\
Snap Star (new warm-up)    & 3 & 0.898±0.002 & 0.285±0.008 & 313.0   \\
Snap Ensemble           & 3 & 0.897±0.003 & 0.290±0.009 & 272.667 \\
Ensemble                & 3 & 0.887±0.001 & 0.310±0.005 & 272.667 \\
Classic Star (no warm-up)  & 3 & 0.893±0.002 & 0.298±0.007 & 339.667 \\
Classic Star (new warm-up) & 3 & 0.893±0.002 & 0.297±0.007 & 285.667 \\
Big NN                  & 3 & 0.890±0.010 & 0.299±0.022 & 214.333 \\
\hline
Snap Star (shot warm-up)   & 2 & \textbf{0.894±0.007} & \textbf{0.294±0.020} & 248.667 \\
Snap Star (new warm-up)    & 2 & 0.892±0.001 & \textbf{0.294±0.006} & 230.333 \\
Snap Ensemble           & 2 & 0.891±0.006 & 0.302±0.021 & 203.667 \\
Ensemble                & 2 & 0.886±0.004 & 0.313±0.008 & 203.0   \\
Classic Star (no warm-up)  & 2 & 0.889±0.003 & 0.304±0.009 & 249.0   \\
Classic Star (new warm-up) & 2 & 0.889±0.004 & 0.303±0.008 & 203.667 \\
Big NN                  & 2 & 0.892±0.003 & 0.304±0.007 & 165.333 \\
\hline
Snap Star (shot warm-up)   & 1 & \textbf{0.891±0.002} & \textbf{0.299±0.006} & 159.0   \\
Snap Star (new warm-up)    & 1 & 0.885±0.001 & 0.318±0.008 & 149.333 \\
Snap Ensemble           & 1 & 0.889±0.001 & 0.304±0.007 & 136.0   \\
Ensemble                & 1 & 0.886±0.005 & 0.314±0.011 & 136.333 \\
Classic Star (no warm-up)  & 1 & 0.888±0.002 & 0.311±0.001 & 158.0   \\
Classic Star (new warm-up) & 1 & \textbf{0.891±0.002} & 0.302±0.005 & 122.333 \\
Big NN                  & 1 & 0.886±0.002 & 0.315±0.005 & 117.333 \\
\hline
\end{tabular}
}
\caption{FASHION MNIST. Part of results at 5 epochs, $lr = 0.001$}
\label{table:fmnist5}
\end{table}
\begin{table}
\centering
\scalebox{1}{
\begin{tabular}{|l|llll|}
\hline
Name                    & d & accuracy    & entropy     & TIME (sec)   \\
\hline
\hline
Snap Star (shot warm-up)   & 5 & 0.898 & 1.152 & 2588.0 \\
Snap Star (new warm-up)    & 5 & 0.898 & 1.136 & 2369.0 \\
Snap Ensemble           & 5 & 0.902 & 0.330 & 2036.0 \\
Ensemble                & 5 & 0.918 & 0.229 & 2052.0 \\
Classic Star (no warm-up)  & 5 & 0.922 & 0.229 & 2589.0 \\
Classic Star (new warm-up) & 5 & \textbf{0.923} & \textbf{0.228} & 2239.0 \\
Big NN                  & 5 & 0.910 & 0.481 & 1560.0\\
\hline
\end{tabular}
}
\caption{FASHION MNIST. All of results at 25 epochs, $lr = 0.001$}
\label{table:fmnist25}
\end{table}

\end{document}